\documentclass{article}
\PassOptionsToPackage{numbers, comma, sort}{natbib}

\usepackage[preprint]{neurips_2019}

\usepackage{amsmath,amsthm,amssymb,mathtools,graphicx,enumitem,hyphenat,float}
\usepackage{mathabx, hyperref}
\usepackage{import}
\usepackage{algorithm, algorithmic}
\newcommand\ev[1]{\left \langle #1 \right \rangle}
\newcommand\br[1]{\left ( #1 \right )}
\newcommand\pbr[1]{\left \{ #1 \right \} }

\newcommand{\N}{\mathbb{N}}

\newcommand{\R}{\mathbb{R}}

\newcommand{\norm}[1]{\left\lVert#1\right\rVert}
\newcommand{\sqn}[1]{{\left\lVert#1\right\rVert}^2}
\newcommand{\abs}[1]{\left\lvert#1\right\rvert}

\newcommand{\eqdef}{\overset{\text{def}}{=}}

\newcommand{\avemm}{\frac{1}{M}\sum_{m=1}^M}

\makeatletter
\newsavebox\myboxA
\newsavebox\myboxB
\newlength\mylenA

\usepackage{tcolorbox}
\usepackage{pifont}
\definecolor{mydarkgreen}{RGB}{39,130,67}

\definecolor{mydarkred}{RGB}{192,47,25}

\newcommand*\overbar[2][0.75]{%
    \sbox{\myboxA}{$\m@th#2$}%
    \setbox\myboxB\null
    \ht\myboxB=\ht\myboxA%
    \dp\myboxB=\dp\myboxA%
    \wd\myboxB=#1\wd\myboxA
    \sbox\myboxB{$\m@th\overline{\copy\myboxB}$}
    \setlength\mylenA{\the\wd\myboxA}
    \addtolength\mylenA{-\the\wd\myboxB}%
    \ifdim\wd\myboxB<\wd\myboxA%
       \rlap{\hskip 0.5\mylenA\usebox\myboxB}{\usebox\myboxA}%
    \else
        \hskip -0.5\mylenA\rlap{\usebox\myboxA}{\hskip 0.5\mylenA\usebox\myboxB}%
    \fi}
\makeatother

\theoremstyle{definition}
\newtheorem{theorem}{Theorem}
\newtheorem{corollary}{Corollary}

\newtheorem{asm}{Assumption}

\newtheorem{lemma}{Lemma}

\usepackage[colorinlistoftodos,bordercolor=orange,backgroundcolor=orange!20,linecolor=orange,textsize=scriptsize]{todonotes}

\theoremstyle{definition}

\begin{document}

\title{First Analysis of Local GD on Heterogeneous Data}
 \author{Ahmed Khaled\thanks{Work done during an internship at KAUST.} \\ Cairo University \\ \texttt{akregeb@gmail.com} \And Konstantin Mishchenko \\ KAUST\thanks{King Abdullah University of Science and Technology, Thuwal, Saudi Arabia} \\ \texttt{konstantin.mishchenko@kaust.edu.sa} \And Peter Richt\'arik \\ KAUST \\ \texttt{peter.richtarik@kaust.edu.sa}}
\maketitle

\begin{abstract} We provide the first convergence analysis of local gradient descent for minimizing the average of smooth and convex but otherwise arbitrary functions. Problems of this form  and local gradient descent as a solution method are of importance in federated learning, where each function is based on private data stored by a user on a mobile device, and the data of different users  can be arbitrarily heterogeneous. We show that in a low accuracy regime, the method has the same communication complexity as gradient descent.
\end{abstract}

\section{Introduction}
We are interested in solving the optimization problem
\begin{equation}
    \label{eq:opt-problem}
  \min_{x \in \R^d} \pbr{ f(x) \eqdef \frac{1}{M} \sum_{m=1}^{M} f_m (x) },
\end{equation}
which is arises in training of supervised machine learning models. We assume that each $f_m:\R^d\to \R$ is an $L$-smooth and convex function and we denote by $x_\ast$ a fixed minimizer of $f$. 

Our main interest is  in situations where each function is based on data available on a single device only, and where the data distribution across the devices can be arbitrarily heterogeneous.  This situation arises in {\em federated learning}, where machine learning models are trained on data available on consumer devices, such as mobile phones. In federated learning, transfer of  local data to a single data center for centralized training is prohibited due to  privacy reasons, and frequent communication is undesirable as it is expensive and intrusive.  Hence, several  recent works aim at constructing new ways of solving \eqref{eq:opt-problem} in a distributed fashion with as few communication rounds as possible.

Large-scale problems are often solved by first-order methods as they have proved to scale well with both dimension and data size. One attractive choice is {\em Local Gradient Descent}, which divides the optimization process into epochs. Each epoch starts by communication in the form of a model averaging step across all $M$ devices.\footnote{In the practice of federated learning, averaging is performed over a subset of devices only. Usually, only those updates are averaged which are received by a certain time window. Here we focus on an idealized scenario where averaging is done across all devices. We focus on this simpler situation first as even this is not currently understood theoretically.  } The rest of each epoch does not involve any communication, and is devoted to performing a fixed number of gradient descent steps initiated from the average model,  and based on the local functions, performed by all $M$ devices independently in parallel. See Algorithm~\ref{alg:local_gd} for more details. 

\begin{algorithm}[t]
   \caption{Local Gradient Descent}
   \label{alg:local_gd}
\begin{algorithmic}[1]
  \REQUIRE Stepsize $\gamma > 0$, synchronization/communication times $0=t_0\leq t_1 \leq t_2 \leq \dots$, initial vector $x_0\in \R^d$
  \STATE {\bf Initialize}  $x_0^m = x_0$ for all $m \in [M]\eqdef \{1,2,\dots,M\}$
   \FOR{$t=0,1,\dotsc$}
      \FOR{$m=1,\dotsc, M$}
         \STATE $x_{t+1}^m=
         \begin{cases}
         \frac{1}{M}\sum_{j=1}^M (x_t^j - \gamma \nabla f_j (x_t^j)), & \text{ if } t = t_p \text { for some } p \in \{1,2,\dots\} \\
         x_t^m - \gamma \nabla f_m (x_t^m), & \text{ otherwise. }
         \end{cases}$
      \ENDFOR
   \ENDFOR
\end{algorithmic}
\end{algorithm}

%

The stochastic version of this method is at the core of the {\em Federated Averaging} algorithm which has been used recently in federated learning applications, see e.g.\ \cite{McMahan17, Konecny16}. Essentially, Federated Averaging is a variant of local Stochastic Gradient Descent (SGD) with participating devices sampled randomly. This algorithm has been used in several machine learning applications such as mobile keyboard prediction~\cite{HardRao18}, and strategies for improving its communication efficiency were explored in~\cite{Konecny16}. Despite its empirical success, little is known about convergence properties of this method and it has been observed to diverge when too many local steps are performed~\cite{McMahan17}. This is not so surprising as the majority of common assumptions are not satisfied; in particular, the data is typically very non-i.i.d.~\cite{McMahan17}, so the local gradients can point in different directions. This property of the data can be written for any vector $x$ and indices $i,j$ as
\begin{align*}
	\|\nabla f_i(x) - \nabla f_j(x)\| \gg 1.
\end{align*}
Unfortunately, it is very hard to analyze local methods without assuming a bound on the dissimilarity of $\nabla f_i(x)$ and $\nabla f_j(x)$. For this reason, almost all prior work assumed bounded dissimilarity~\cite{yu2019linear, Li2019, Yu18, WangTuor18} and addressed other less challenging aspects of federated learning such as decentralized communication, nonconvexity of the objective or unbalanced data partitioning. In fact, a common way to make the analysis simple is to assume Lipschitzness of local functions,
$
	\|\nabla f_i(x)\|\le G
$
for any $x$ and $i$. We argue that this assumption is pathological and should be avoided when seeking a meaningful convergence bound. First of all, in unconstrained strongly convex minimization this assumption cannot be satisfied, making the analysis in works like~\cite{Stich2018} questionable. Second, there exists at least one method, whose convergence is guaranteed under bounded gradients~\cite{juditsky2011solving}, but in practice the method diverges~\cite{chavdarova2019reducing, mishchenko2019revisiting}. 

Finally, under the bounded gradients assumption we have
\begin{align}
    \label{eq:related-work-2}
    \norm{\nabla f_i (x) - \nabla f_{j} (x)} \leq \norm{\nabla f_i (x)} + \norm{\nabla f_j (x)} \leq 2G.
\end{align}
In other words, we lose control over the difference between the functions. Since $G$ bounds not just dissimilarity, but also the gradients themselves, it makes the statements less insightful or even vacuous. For instance, it is not going to be tight if the data is actually i.i.d.\ since $G$ in that case will remain a positive constant. In contrast, we will show that the rate should depend on a much more meaningful quantity, 
$$
	\sigma^2\eqdef \frac{1}{M}\sum_{m=1}^M \|\nabla f_m (x_*)\|^2,
$$
where $x_*$ is a minimizer of $f$. Obviously, $\sigma$ is always finite and it serves as a natural measure of variance in local methods. On top of that, it allows us to obtain bounds that are tight in case the data is actually i.i.d. We note that an attempt to get more general convergence statement has been made in~\cite{Sahu18}, but sadly their guarantee is strictly worse than that of minibatch Stochastic Gradient Descent (SGD), making their theoretical contribution smaller.

We additionally note that the bound in the mentioned work~\cite{Li2019} not only uses bounded gradients, but also provides a pessimistic ${\cal O} (H^2/T)$ rate, where $H$ is the number of local steps in each epoch, and $T$ is the total number of steps of the method. Indeed, this requires $H$ to be ${\cal O} (1)$ to make the rate coincide with that of SGD for strongly convex functions. The main contribution of that work, therefore, is in considering partial participation as in Federated Averaging.

When the data is identically distributed and stochastic gradients are used instead of full gradients on each node, the resulting method has been explored extensively in the literature under different names, see e.g.\ \cite{Stich2018, Basu2019, Wang18, Zhou18}. \cite{mishchenko2018delay} proposed an asynchronous local method that converges to the exact solution without decreasing stepsizes, but its benefit from increasing $H$ is limited by constant factors. \cite{Mangasarian95} seems to be the first work to propose a local method, but no rate was shown in that work.

\section{Convergence of Local GD}
\subsection{Assumptions and notation}
Before introducing our main result, let us first formulate explicitly our assumptions. 
\begin{asm}
    \label{asm:smoothness-and-convexity}
    The set of minimizers of \eqref{eq:opt-problem} is nonempty.
Further, for every $m\in [M]\eqdef \{1,2,\dots,M\}$, $f_m$ is convex and $L$-smooth. That is, for all $x, y\in \R^d$ the following inequalities are satisfied:
$$
	    0\le f_m (x) - f_m (y) - \ev{\nabla f_m (y), x - y} \leq \tfrac{L}{2}\|x-y\|^2.
$$
\end{asm}
 Further, we assume that Algorithm~\ref{alg:local_gd} is run with a bounded  synchronization interval. That is, we assume that $$H \eqdef \max_{p\geq 0} \abs{t_p - t_{p+1}}$$ is finite.  Given local vectors $x_t^1, x_t^2, \ldots, x_t^M \in \R^d$, we define the average iterate, iterate variance  and average gradient at time $t$ as
\begin{align}
    \hat{x}_t \eqdef \avemm x_t^m &&  
    V_t \eqdef \avemm \sqn{x_t^m - \hat{x}_t} &&  
    g_t \eqdef \frac{1}{M} \sum_{m=1}^{M} \nabla f_m (x_t^m),
\end{align}
respectively. The Bregman divergence with respect to $f$ is defined via
$$
    D_{f} (x, y) \eqdef f(x) - f(y) - \ev{\nabla f(y), x - y}.$$ 
Note that in the case $y = x_\ast$, we have
$
    D_{f} (x, x_\ast) = f(x) - f(x_\ast).
$


\subsection{Analysis}

The first lemma enables us to find a recursion on the optimality gap for a single step of local GD:
\begin{lemma}
    \label{lemma:optimality-gap-recursion}
    Under Assumption~\ref{asm:smoothness-and-convexity} and for any $\gamma \geq 0$ we have
    \begin{equation}
        \label{eq:9f8gff}
        \sqn{r_{t+1}} \leq  \sqn{r_t} + \gamma L \br{1 + 2 \gamma L} V_t - 2 \gamma \br{1 - 2 \gamma L} D_{f} (\hat{x}_t, x_\ast),
    \end{equation}   
    where $r_{t} \eqdef \hat{x}_t - x_\ast$. In particular, if $\gamma \leq \frac{1}{4L}$, then
$
        \sqn{r_{t+1}} \leq \sqn{r_t} + \tfrac{3}{2} \gamma L V_t - \gamma D_{f} (\hat{x}_t, x_\ast).
$
\end{lemma}

We now bound  the sum of the variances $V_t$ \textit{over an epoch}. An epoch-based bound is intuitively what we want  since we are only interested in the points $\hat{x}_{t_p}$ produced at the end of each epoch.

\begin{lemma}
    \label{lemma:Vt-bound}
    Suppose that Assumption~\ref{asm:smoothness-and-convexity} holds and let $p \in \N$, define $v = t_{p+1} - 1$ and suppose Algorithm~\ref{alg:local_gd} is run with a synchronization interval $H\geq 1$ and a constant stepsize $\gamma > 0$ such that $\gamma \leq \frac{1}{4 L H}$. Then the following inequalities hold:
    \begin{align*}
  &\sum_{t=t_p}^{v} V_t \leq 5 L \gamma^2 H^2 \sum_{i=t_p}^{v} D_{f} (\hat{x}_i, x_\ast) + \sum_{i=t_p}^{v} 8 \gamma^2 H^2 \sigma^2,\\
      &\sum_{t=t_p}^{v} \frac{3}{2} L V_t - D_{f} (\hat{x}_t, x_\ast)       \leq  -\frac{1}{2} \sum_{t=t_p}^{v} D_{f} (\hat{x}_i, x_\ast) + \sum_{t=t_p}^{v} 12 L \gamma^2 H^2 \sigma^2.
    \end{align*}
\end{lemma}

Combining the previous two lemmas,  the convergence of local GD is established in the next theorem:
\begin{theorem}
    \label{theorem:local-gd-weak-convexity}
    For local GD run with a constant stepsize $\gamma > 0$ such that $\gamma \leq \frac{1}{4 L H}$ and under Assumption~\ref{asm:smoothness-and-convexity}, we have
    \begin{align}
        \label{eq:thm-cnvx-local-gd}
        f(\bar{x}_T) - f(x_\ast) &\leq \frac{2 \sqn{x_0 - x_\ast}}{\gamma T} + 24 \gamma^2 \sigma^2 H^2 L,
    \end{align}
    where $\bar{x}_T \eqdef \frac{1}{T} \sum_{t=0}^{T-1} \hat{x}_{t}$.
\end{theorem}


\subsection{Local GD vs GD} 

In order to interpret the above bound,  we may ask: how many communication rounds are sufficient to guarantee $ f(\bar{x}_T) - f(x_\ast) \leq \epsilon$? To answer this question, we need to minimize $\frac{T}{H}$ subject to the constraints $0<\gamma \leq \frac{1}{4L}$, $\frac{2 \|x_0-x_\ast\|^2}{\gamma T} \leq \frac{\epsilon}{2} $, and $24 \gamma^2 \sigma^2 H^2 L \leq  \frac{\epsilon}{2}$, in variables $T, H$ and $\gamma$. We can easily deduce from the constraints that \begin{equation}\label{eq:coparison_GD}\frac{T}{H}\geq \frac{16  \|x_0-x_\ast\|^2}{\epsilon}\max \left\{L, \sigma \sqrt{\frac{3  L}{\epsilon}}\right\}.\end{equation} On the other hand, this lower bound is achieved by {\em any} $0<\gamma\leq \frac{1}{4L}$ as long as we pick $$T = T(\gamma)\eqdef \frac{4  \|x_0-x_\ast\|^2}{\epsilon \gamma}\qquad \text{and} \qquad H = H(\gamma)\eqdef \frac{1}{4\max \left\{L, \sigma \sqrt{\frac{3  L}{\epsilon}}\right\} \gamma}.$$  The smallest $H$ achieving this lower bound is $H(\frac{1}{4L}) = \min\left\{1,\sqrt{\frac{\epsilon L }{3 \sigma^2}}\right\}$. 

Further, notice that as long as 
the target accuracy is not too high, in particular $\epsilon \geq \frac{3\sigma^2}{L}$, then $ \max \left\{L, \sigma \sqrt{3  L/\epsilon}\right\}=L$ and \eqref{eq:coparison_GD} says that the number of communications of local GD (with parameters set as $H=H(\gamma)$ and $T=T(\gamma)$) is equal to $$\frac{T}{H} = {\cal O} \left(  \frac{L \|x_0-x_\ast\|^2}{\epsilon} \right),$$ {\em which is the same as the number of iterations (i.e., communications) of gradient descent.} 
If $\epsilon < \frac{3\sigma^2}{L}$, then \eqref{eq:coparison_GD} gives the communication complexity $$\frac{T}{H} = {\cal O}\left( \frac{ \sqrt{L} \sigma}{ \epsilon^{3/2}}\right).$$


\subsection{Local GD vs Minibatch SGD} \label{sec:mSGD}

Equation~\eqref{eq:thm-cnvx-local-gd} shows a clear analogy between the convergence of local GD and the convergence rate of minibatch SGD, establishing a $1/T$ convergence to a neighborhood depending on the expected noise at the optimum $\sigma^2$, which measures how dissimilar the functions $f_m$ are from each other at the optimum $x_\ast$. 

The analogy between SGD and local GD extends further to the convergence rate, as the next corollary shows:

\begin{corollary}
    Choose $H$ such that $H \leq \frac{\sqrt{T}}{\sqrt{M}}$, then $\gamma = \frac{\sqrt{M}}{4 L \sqrt{T}} \leq \frac{1}{4 H L}$, and hence applying the result of the previous theorem
    \begin{align*}
        f(\bar{x}_T) - f(x_\ast) &\leq \frac{8 L \sqn{x_0 - x_\ast}}{\sqrt{M T}} + \frac{3 M \sigma^2 H^2}{2 L T}.
    \end{align*}
    To get a convergence rate of $1/\sqrt{MT}$ we can choose $H = O\br{T^{1/4} M^{-3/4}}$, which implies a total number of $\Omega(T^{3/4} M^{3/4})$ communication steps. If a rate of $1/\sqrt{T}$ is desired instead, we can choose a larger $H = O\br{T^{1/4}}$.
\end{corollary}

\section{Experiments}
To verify the theory, we run our experiments on logistic regression with $\ell_2$ regularization and datasets taken from the LIBSVM library~\cite{chang2011libsvm}. We use a machine with 24 Intel(R) Xeon(R) Gold 6146 CPU @ 3.20GHz cores and we handle communication via the MPI for Python package~\cite{dalcin2011parallel}. 

Since our architecture leads to a very specific trade-off between computation and communication, we also provide plots for the case the communication time relative to gradient computation time is higher or lower. In all experiments, we use full gradients $\nabla f_m$ and constant stepsize $\frac{1}{L}$. The amount of $\ell_2$ regularization was chosen of order $\frac{1}{n}$, where $n$ is the total amount of data. The data partitioning is not i.i.d. and is done based on the index in the original dataset.

We observe a very tight match between our theory and numerical results. In cases where communication is significantly more expensive than gradient computation, local methods are much faster for imprecise convergence. This was not a big advantage though with our architecture, mainly because full gradients took a lot of time to be computed.

\begin{figure}[!th]
	\centering
	\includegraphics[scale=0.2]{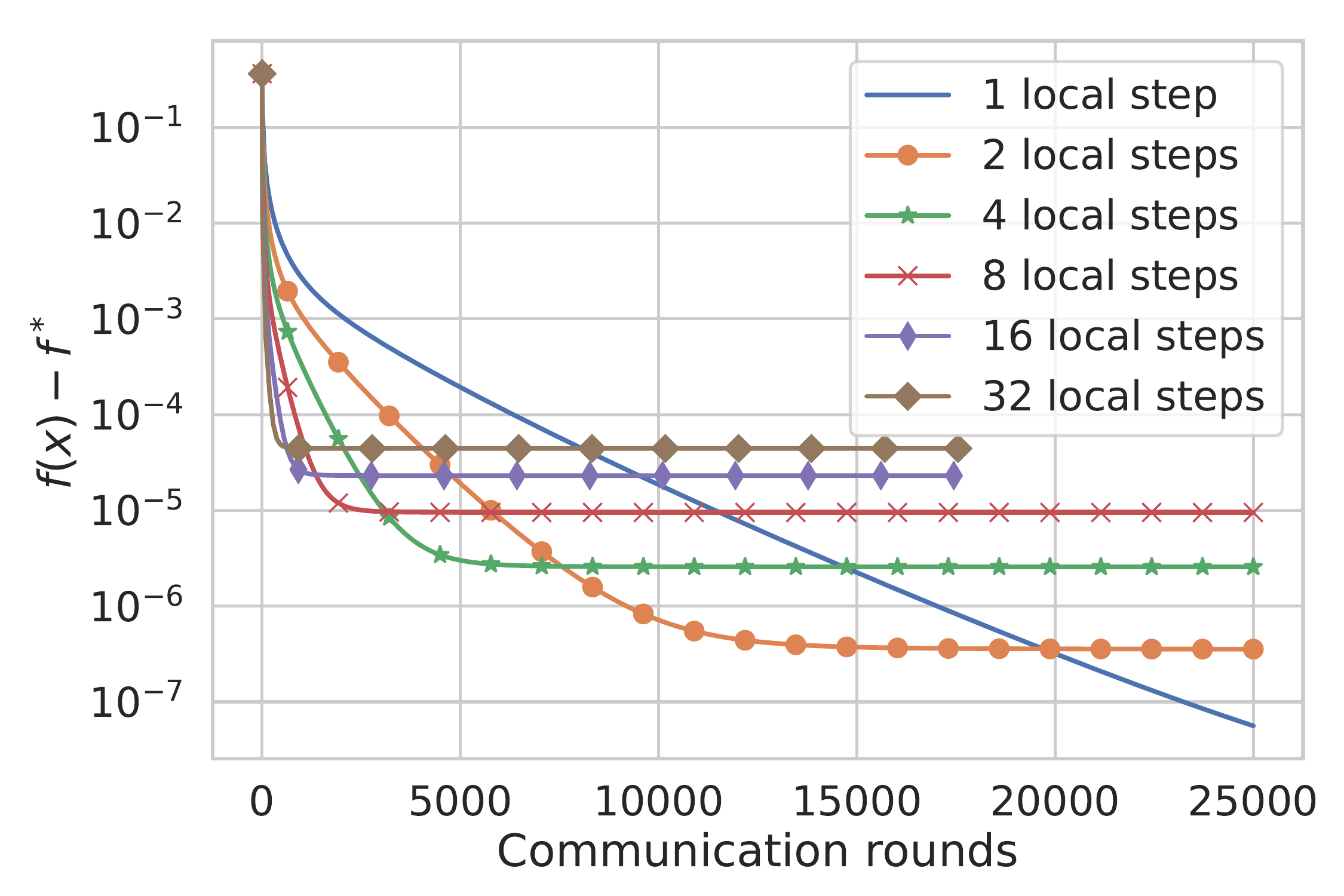}
	\includegraphics[scale=0.2]{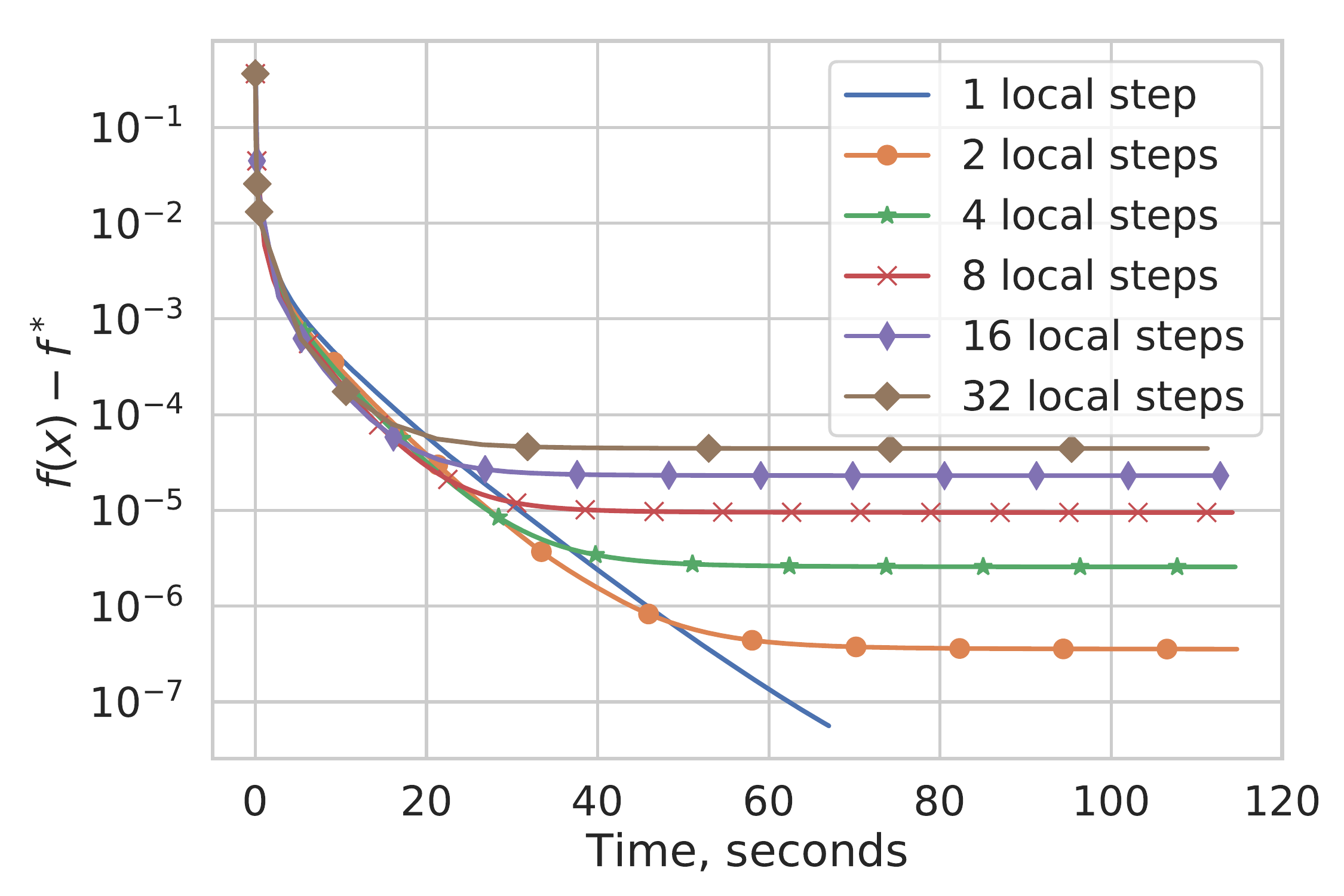}
	\includegraphics[scale=0.2]{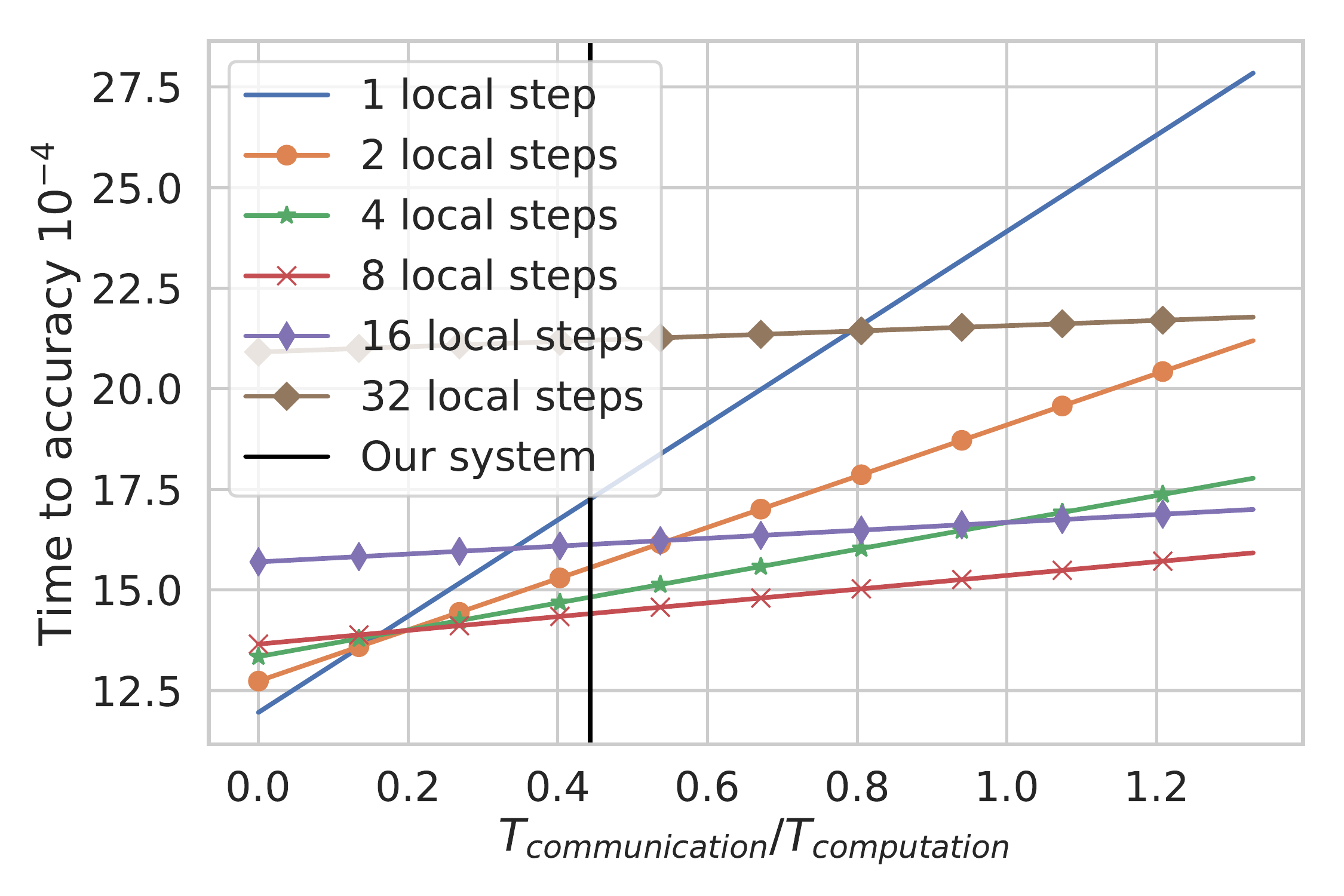}
	\label{fig:a5a_different_H}
	\caption{Convergence of local GD methods with different number of local steps on the 'a5a' dataset. 1 local step corresponds to fully synchronized gradient descent and it is the only method that converges precisely to the optimum. The left plot shows convergence in terms of communication rounds, showing a clear advantage of local GD when only limited accuracy is required. The mid plot, however, illustrates that wall-clock time might improve only slightly and the right plot shows what changes with different communication cost.}
\end{figure}

\begin{figure}[!th]
	\centering
	\includegraphics[scale=0.2]{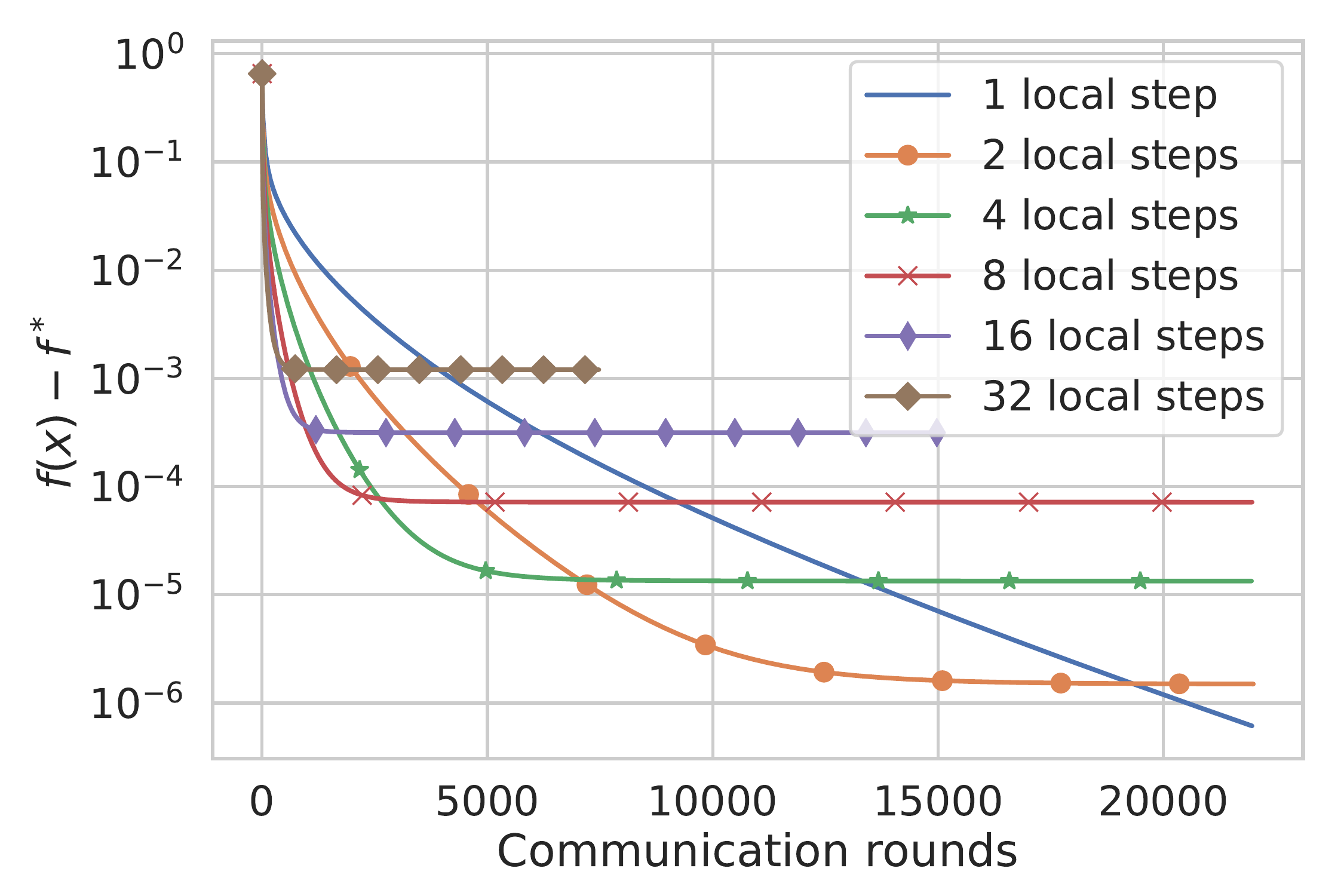}
	\includegraphics[scale=0.2]{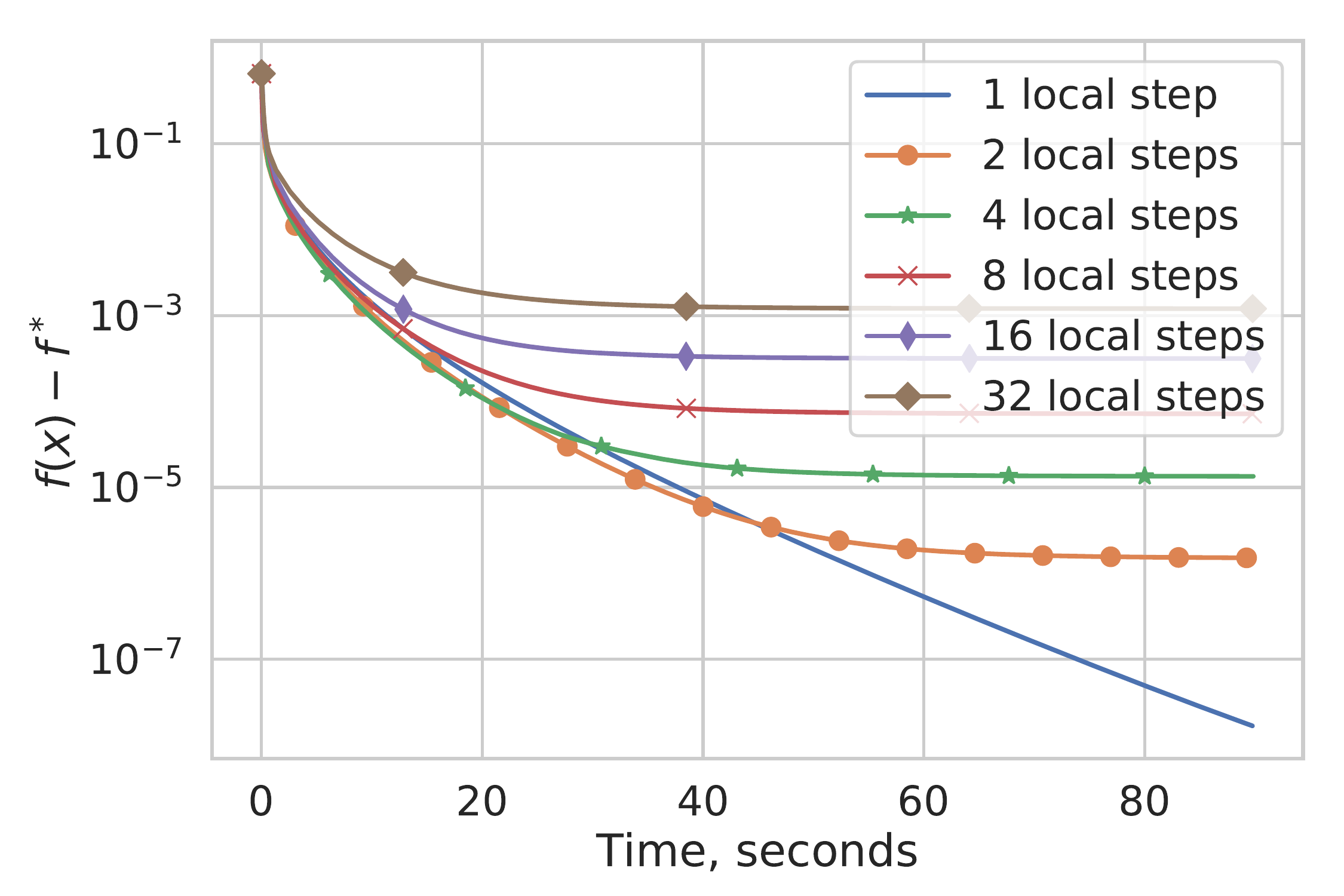}
	\includegraphics[scale=0.2]{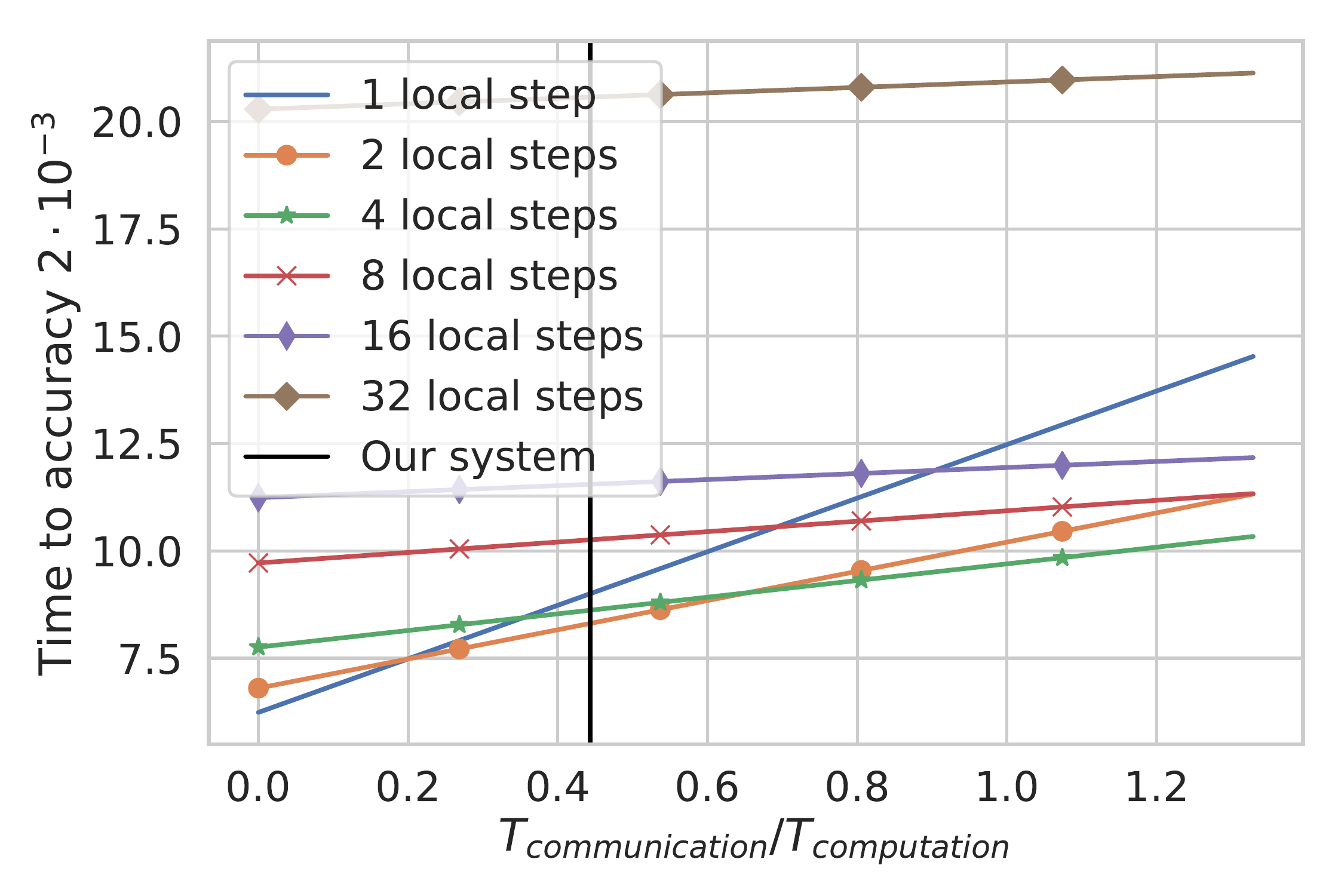}
	\label{fig:mushrooms_different_H}
	\caption{Same experiment as in Figure~\ref{fig:a5a_different_H}, performed on the 'mushrooms' dataset.}
\end{figure}

\clearpage
\bibliography{local_gd}
\clearpage

\part*{Supplementary Material for:\\  First Analysis of Local GD on Heterogeneous Data}

\section{Proofs}

We first provide two technical lemmas which relate the quantities  $\hat{x}_t$, $V_t$ and $g_t$, $x_\ast$ and $\nabla f_m(x_t^m)$ for $m=1,2,\dots,M$. These lemmas are independent of the algorithm.

\begin{lemma}
    \label{lemma:average-gradient-bound}
    If Assumption~\ref{asm:smoothness-and-convexity} holds, then
    \begin{align}
        \label{eq:lma-average-gradient-bound}
        \sqn{g_t} &\leq 2 L^2 V_t + 4 L D_{f} (\hat{x}_t, x_\ast).
    \end{align}
\end{lemma}
\begin{proof}
    Starting with the left-hand side,
    \begin{align*}
        \sqn{g_t} &\leq 2 \sqn{g_t - \nabla f(\hat{x}_t)} + 2 \sqn{\nabla f(\hat{x}_t)} \\
        &= 2 \sqn{ \avemm \nabla f_m (x_t^m) - \avemm \nabla f_m (\hat{x}_t) } + 2 \sqn{\nabla f(\hat{x}_t)} \\
        &\leq \frac{2}{M} \sum_{m=1}^{M} \sqn{\nabla f_m (x_t^m) - \nabla f_m (\hat{x}_t)} + 2 \sqn{\nabla f(\hat{x}_t)} \\
        &\leq \frac{2 L^2}{M} \sum_{m=1}^{M} \sqn{x_t^m - \hat{x}_t} + 2 \sqn{\nabla f(\hat{x}_t)},
    \end{align*}
    where in the second inequality we have used convexity of the map $x\mapsto \|x\|^2$. The claim of the lemma follows by noting that
    \begin{align*}
        \sqn{\nabla f(\hat{x}_t)} = \sqn{\nabla f(\hat{x}_t) - \nabla f(x_\ast)} \leq 2 L D_{f} (\hat{x}_t, x_\ast).
    \end{align*}
   
\end{proof}

\begin{lemma}
    \label{lemma:inner-product-bound}
    Suppose that Assumption~\ref{asm:smoothness-and-convexity} holds. Then,
    \begin{equation}
        \label{eq:lma-inner-product-bound}
        -\frac{2}{M} \sum_{m=1}^{M} \ev{\hat{x}_t - x_\ast, \nabla f_m (x_t^m)} \leq - 2  D_{f} (\hat{x}_t, x_\ast) +  L V_t.
    \end{equation}
\end{lemma}
\begin{proof}
    Starting with the left-hand side,
    \begin{align}
        \label{eq:lma-inner-prod-proof-1}
        - 2  \ev{\hat{x}_t - x_\ast, \nabla f_m (x_t^m)} &= - 2  \ev{x_t^m - x_\ast, \nabla f_m (x_t^m)} - 2  \ev{\hat{x}_t - x_t^m, \nabla f_m (x_t^m)}.
    \end{align}
    The first term in \eqref{eq:lma-inner-prod-proof-1} is bounded by convexity:
    \begin{align}
        \label{eq:lma-inner-prod-proof-2}
        - \ev{x_t^m - x_\ast, \nabla f_m (x_t^m)} &\leq f_m (x_\ast) - f_m (x_t^m).
    \end{align}
    For the second term, we use $L$-smoothness,
    \begin{align}
        \label{eq:lma-inner-prod-proof-3}
        - \ev{\hat{x}_t  - x_t^m, \nabla f_m (x_t^m)} \leq f_m (x_t^m) - f_m (\hat{x}_t) + \frac{L}{2} \sqn{x_t^m - \hat{x}_t}.
    \end{align}
    Combining \eqref{eq:lma-inner-prod-proof-3} and \eqref{eq:lma-inner-prod-proof-2} in \eqref{eq:lma-inner-prod-proof-1},
    \begin{align*}
        - 2  \ev{\hat{x}_t - x_\ast, \nabla f_m (x_t^m)} &\leq 2  \br{f_m (x_\ast) - f_m (x_t^m)} \\
        &+ 2  \br{f_m (x_t^m) - f_m (\hat{x}_t) + \frac{L}{2} \sqn{x_t^m - \hat{x}_t}} \\
        &= 2  \br{f_m (x_\ast) - f_m (\hat{x}_t) + \frac{L}{2} \sqn{x_t^m - \hat{x}_t}}.
    \end{align*}
    Averaging over $m$,
    \begin{align*}
        -\frac{2 }{M} \sum_{m=1}^{M} \ev{\hat{x}_t - x_\ast, \nabla f_m (x_t^m)} &\leq - 2  \br{f(\hat{x}_t) - f(x_\ast)} + \frac{ L}{M} \sum_{m=1}^{M} \sqn{x_t^m - \hat{x}_t}.
    \end{align*}
    Note that the first term is the Bregman divergence $D_{f} (\hat{x}_t, x_\ast)$ and the second term is just $V_t$, hence
    \begin{align*}
        -\frac{2 }{M} \sum_{m=1}^{M} \ev{\hat{x}_t - x_\ast, \nabla f_m (x_t^m)} &\leq - 2  D_{f} (\hat{x}_t, x_\ast) +  L V_t,
    \end{align*}
    which is the claim of this lemma.
\end{proof}

\begin{proof}[\textbf{Proof of Lemma~\ref{lemma:optimality-gap-recursion}}]
    Note that $\hat{x}_{t+1} = \hat{x}_t - \gamma g_t$ always holds. Then we have,
    \begin{align*}
        \sqn{\hat{x}_{t+1} - x_\ast} &= \sqn{\hat{x}_t - \gamma g_t - x_\ast} \\
        &= \sqn{\hat{x}_t - x_\ast} + \gamma^2 \sqn{g_t} - 2 \gamma \ev{\hat{x}_t - x_\ast, g_t} \\
        &= \sqn{\hat{x}_t - x_\ast} + \gamma^2 \sqn{g_t} - \frac{2 \gamma}{M} \sum_{m=1}^{M} \ev{\hat{x}_t - x_\ast, \nabla f_m (x_t^m)}
    \end{align*}
    Let $r_{t} = \hat{x}_t - x_\ast$, then using Lemmas~\ref{lemma:average-gradient-bound} and \ref{lemma:inner-product-bound},
    \begin{align*}
        \sqn{r_{t+1}} &\leq \sqn{r_t} + \gamma^2 \sqn{g_t} - \frac{2 \gamma}{M} \sum_{m=1}^{M} \ev{\hat{x}_t - x_\ast, \nabla f_m (x_t^m)} \\
        &\overset{\eqref{eq:lma-average-gradient-bound}}{\leq} \sqn{r_t} + \gamma^2 \br{ 2 L^2 V_t + 4 L D_{f} (\hat{x}_t, x_\ast) } - \frac{2 \gamma}{M} \sum_{m=1}^{M} \ev{\hat{x}_t - x_\ast, \nabla f_m (x_t^m)}  \\
        &\overset{\eqref{eq:lma-inner-product-bound}}{\leq} \sqn{r_t} + \gamma L \br{1 + 2 \gamma L} V_t - 2 \gamma \br{1 - 2 \gamma L} D_{f} (\hat{x}_t, x_\ast).
    \end{align*}
    If $\gamma \leq \frac{1}{4L}$, then $1 - 2 \gamma L \geq \frac{1}{2}$ and $1 + 2 \gamma L \leq \frac{3}{2}$, and hence
    \begin{align*}
        \sqn{r_{t+1}} &\leq \sqn{r_t} + \frac{3}{2} \gamma L V_t - \gamma D_{f} (\hat{x}_t, x_\ast),
    \end{align*}
    as claimed.
\end{proof}

\begin{proof}[\textbf{Proof of Lemma~\ref{lemma:Vt-bound}}]
    Let $g_t^m = \nabla f(x_t^m)$, then noting that $x_{t+1}^m = x_t^m - \gamma g_t^m$ when $t_{p+1} > t > t_p$ and $x_{t_p}^m = \hat{x}_{t_p}$ we have, 
    \begin{align*}
        V_{t} &= \avemm \sqn{x_t^m - \hat{x}_t} = \avemm \sqn{ x_{t_p}^{m} - \hat{x}_{t_p} - \gamma \sum_{i=t_p}^{t-1} g_i^m - g_i } \\
        &= \frac{\gamma^2}{M} \sum_{m=1}^{M} \sqn{\sum_{i=t_p}^{t-1} \br{g_i^m - g_i}} \\
        &\leq \frac{\gamma^2}{M} \sum_{m=1}^{M} \br{t - t_p} \sum_{i=t_p}^{t-1} \sqn{g_i^m - g_i} \\
        &\leq \frac{\gamma^2 H}{M} \sum_{m=1}^{M} \sum_{i=t_p}^{t-1} \sqn{g_i^m - g_i}  \\
        &\leq \frac{\gamma^2 H}{M} \sum_{m=1}^{M} \sum_{i=t_p}^{t-1} \sqn{g_i^m} .
    \end{align*}
    Then we have 
    \begin{align*}
        \sqn{g_i^m} &\leq \br{1 + \alpha} \sqn{g_i^m - \nabla f_m (\hat{x}_i)} + \br{1 + \alpha^{-1}} \sqn{\nabla f_m (\hat{x}_i)} \\
        &\leq \br{1 + \alpha}\sqn{g_i^m - \nabla f_m (\hat{x}_i)} + \br{1 + \alpha^{-1}} \br{1 + \beta} \sqn{\nabla f_m (\hat{x}_i) - \nabla f_m (x_\ast)} \\
        &+ \br{1 + \alpha^{-1}} \br{1 + \beta^{-1}} \sqn{\nabla f_m (x_\ast)}.
    \end{align*}
    Putting $\alpha = 2$, $\beta = \frac{1}{3}$, we get
    \begin{align*}
        \sqn{g_i^m} &\leq 3 \sqn{g_i^m - \nabla f_m (\hat{x}_i)} + 2 \sqn{\nabla f_m (\hat{x}_i) - \nabla f_m (x_\ast)} + 6 \sqn{\nabla f_m (x_\ast)} \\
        &\leq 3 L^2 \sqn{x_t^m - \hat{x}_t} + 2 \br{2 L D_{f_m} (\hat{x}_t, x_\ast)} + 6 \sqn{\nabla f_m (x_\ast)}.
    \end{align*}
    Averaging with respect to $m$,
    \begin{align*}
        \avemm \sqn{g_t^m} &\leq 3 L^2 V_t + 4 L D_{f} (\hat{x}_t, x_\ast) + 6 \sigma^2.
    \end{align*}
    Hence we have,
    \begin{align*}
        V_t &\leq \gamma^2 H \sum_{i=t_p}^{t} \frac{1}{M} \sum_{m=1}^{M} \sqn{g_t^m} \leq \gamma^2 H \sum_{i=t_p}^{t} \br{3 L^2 V_i + 4 L D_{f} (\hat{x}_i, x_\ast) + 6 \sigma^2 }.
    \end{align*}
    Summing up the above inequality as $t$ varies from $t_p$ to $v = t_{p+1} - 1$,
    \begin{align*}
        \sum_{t=t_p}^{v} V_t &\leq \gamma^2 H \sum_{t=t_p}^{v} \sum_{i=t_p}^{t} \br{3 L^2 V_i + 4 L D_{f} (\hat{x}_i, x_\ast) + 6 \sigma^2 } \\
        &\leq \gamma^2 H \sum_{t=t_p}^{v} \sum_{i=t_p}^{v} \br{3 L^2 V_i + 4 L D_{f} (\hat{x}_i, x_\ast) + 6 \sigma^2 } \\
        &= 3 L^2 \gamma^2 H^2 \sum_{i=t_p}^{v} V_i + 4 L \gamma^2 H^2 \sum_{i=t_p}^{v} D_{f} (\hat{x}_i, x_\ast) + \sum_{i=t_p}^{v} 6 \gamma^2 H^2 \sigma^2.
    \end{align*}
    Noting that the sum $\sum_{t=t_p}^{v} V_t$ appears on both sides, we have
    \begin{align*}
        \br{1 - 3 L^2 \gamma^2 H^2} \sum_{t=t_p}^{v} V_t &\leq 4 L \gamma^2 H^2 \sum_{i=t_p}^{v} D_{f} (\hat{x}_i, x_\ast) + 6 \gamma^2 H^2 \sigma^2.
    \end{align*}
    Note that because $\gamma \leq \frac{1}{4 L H} \leq \frac{1}{\sqrt{15} L H}$, then our choice of $\gamma$ implies that $1 - 3 L^2 \gamma^2 H^2 \geq \frac{4}{5}$, hence
    \begin{align*}
        \sum_{t=t_p}^{v} V_t \leq 5 L \gamma^2 H^2 \sum_{i=t_p}^{v} D_{f} (\hat{x}_i, x_\ast) + \sum_{i=t_p}^{v} \frac{15}{2} \gamma^2 H^2 \sigma^2.
    \end{align*}
    For the second part, we have
    \begin{align*}
        \sum_{t=t_p}^{v} \frac{3}{2} L V_t - \sum_{i=t_p}^{v} D_{f} (\hat{x}_i, x_\ast) \leq \br{\frac{15}{2} L^2 \gamma^2 H^2 - 1} \sum_{i=t_p}^{v} D_{f} (\hat{x}_i, x_\ast) + \sum_{i=t_p}^{v} \frac{45}{4} L \gamma^2 H^2 \sigma^2.
    \end{align*}
    Using that $\gamma \leq \frac{1}{4 L H}$ we get that $\frac{15}{2} L^2 \gamma^2 H^2 - 1 \leq \frac{-1}{2}$, and using this we get the second claim.
\end{proof}

\begin{proof}[\textbf{Proof of Theorem~\ref{theorem:local-gd-weak-convexity}}]
    Starting with Lemma~\ref{lemma:optimality-gap-recursion}, we have
    \begin{align*}
        \sqn{r_{t+1}} &\leq \sqn{r_t} + \gamma \br{\frac{3}{2} L V_t - D_{f} (\hat{x}_t, x_\ast)} \leq \sqn{r_t} + \gamma \br{ 2 L V_t - D_{f} (\hat{x}_t, x_\ast) }.
    \end{align*}
    Summing up these inequalities gives
    \begin{align*}
        \sum_{i=1}^{T} \sqn{r_t} &\leq \sum_{i=0}^{T-1} \sqn{r_t} + \gamma \sum_{i=0}^{T-1} \br{2 L V_i - D_{f} (\hat{x}_i, x_\ast)},
    \end{align*}
and     using that $T = t_p$ for some $p \in \N$, we can decompose the second term by double counting and bound it by Lemma~\ref{lemma:Vt-bound}, then using double counting again,
    \begin{align*}
        \gamma \sum_{i=0}^{T-1} \br{2 L V_i - D_{f} (\hat{x}_i, x_\ast)} &= \gamma \sum_{k=1}^{p} \sum_{i=t_{k-1}}^{t_{k} - 1} \br{2 L V_i - D_{f} (\hat{x}_i, x_\ast)} \\
        &\leq -\frac{\gamma}{2} \sum_{k=1}^{p} \sum_{i=t_p}^{v} D_{f} (\hat{x}_i, x_\ast) + \sum_{k=1}^{p} \sum_{i=t_p}^{v} 12 L \gamma^3 H^2 \sigma^2 \\
        &= -\frac{\gamma}{2} \sum_{i=0}^{T-1} D_{f} (\hat{x}_i, x_\ast) + \sum_{i=0}^{T-1} 12 L \gamma^3 H^2 \sigma^2.
    \end{align*}
    Using this in the previous bound,
    \begin{align*}
        \sum_{i=1}^{T} \sqn{r_t} &\leq \sum_{i=0}^{T-1} \sqn{r_t} - \frac{\gamma}{2} \sum_{i=0}^{T-1} D_{f} (\hat{x}_i, x_\ast) + \sum_{i=0}^{T-1} 12 L \gamma^3 H^2 \sigma^2.
    \end{align*}
    Rearranging we get,
    \begin{align*}
        \frac{\gamma}{2} \sum_{i=0}^{T-1} D_f (\hat{x}_i, x_\ast) &\leq \sum_{i=0}^{T-1} \sqn{r_t} - \sum_{i=1}^{T} \sqn{r_t} + \sum_{i=0}^{T-1} 12 L \gamma^3 H^2 \sigma^2 \\
        &= \sqn{r_0} - \sqn{r_T} + \sum_{i=0}^{T-1} 12 L \gamma^3 H^2 \sigma^2 \\
        &\leq \sqn{r_0} + 12 T L \gamma^3 H^2 \sigma^2.
    \end{align*}
    Dividing both sides by $\gamma T / 2$ we get,
    \begin{align*}
        \frac{1}{T} \sum_{i=0}^{T-1} D_f (\hat{x}_i, x_\ast) &\leq \frac{2 r_0}{\gamma T} + 24 L \gamma^2 H^2 \sigma^2.
    \end{align*}
    Finally, using Jensen's inequality and the convexity of $f$ we get the required claim.
\end{proof}

\end{document}